\documentclass[11pt,twocolumn,letterpaper]{article}

\usepackage{microtype}
\usepackage{graphicx}
\usepackage{amsmath}
\usepackage{subfigure}
\usepackage{booktabs} %
\usepackage{amsthm}
\usepackage[english]{babel}
 \usepackage[utf8]{inputenc}
\newtheorem{theorem}{Theorem}[section]

\newtheorem{lemma}[theorem]{Lemma}
\usepackage{amsmath, amssymb}
\newcommand{\norm}[1]{\left\lVert#1\right\rVert}

\usepackage{xcolor, eufrak}
\theoremstyle{definition}
\newtheorem{definition}{Definition}[section]
\theoremstyle{remark}

\usepackage{mathtools}
\newcommand{\given}{\,|\,}
\usepackage{tikz}
\usetikzlibrary{arrows.meta}
\DeclareMathOperator*{\argmin}{arg\,min}
\usepackage{authblk}
\usepackage{cite}
\usepackage{hyperref}
\usepackage{tikz}
\usepackage{tikzscale}
\usepackage[square,numbers]{natbib}
\usepackage{microtype}

\usepackage[top=0.7in,bottom=0.7in,left=0.6in,right=0.7in,columnsep=0.35in]{geometry}

\title{First Order Methods take Exponential Time to Converge to Global Minimizers of Non-Convex Functions}
\author[1]{Krishna Reddy Kesari}
\author[2]{Jean Honorio}
\affil[1]{Department of Electrical and Computer Engineering, Purdue University}
\affil[2]{Department of Computer Science, Purdue University}

\begin{document}

\maketitle

\begin{abstract}
Machine learning algorithms typically perform optimization over a class of non-convex functions. In this work, we provide bounds on the fundamental hardness of identifying the global minimizer of a non convex function. Specifically, we design a family of parametrized non-convex functions and employ statistical lower bounds for parameter estimation. We show that the parameter estimation problem is equivalent to the problem of function identification in the given family. We then claim that non convex optimization is at least as hard as function identification. Jointly, we prove that any first order method can take exponential time to converge to a global minimizer.   
\end{abstract}

\section{Introduction}
\label{intro}
Deep learning algorithms generally employ first order optimization techniques that have convergence rates established only for convex functions. However, the function classes that they model happen to be non convex in nature. Non convex functions have epigraphs that are not convex sets. In other words, they do not exhibit the characteristic of having a single local minimum that also serves as its global minimum. Instead, they may have multiple local and multiple global minima. Although identifying global minimizers is not the goal of learning algorithms, it has motivated recent work for a better understanding of non convex optimization. Specifically, there has been emphasis on saddle points and local minimizers as understanding the geometry around these points may enable better algorithms to scan the search space. It was first shown that gradient descent converges to minimizers if the strict saddle property is satisfied \cite{jmlr_lee16}. Subsequent work has shown that gradient descent with a constant step size can take exponential time to escape saddle points \cite{Du2017}. On the other hand, perturbed gradient descent is capable of escaping these saddle points under some technical assumptions including smoothness, Lipschitz Hessian, strict saddle points, among others \cite{jin}. While this line of work attempts to provide guarantees under specific conditions, it is difficult in most machine learning scenarios to apriori understand which of these conditions are satisfied. In this work, we take an orthogonal approach to understanding the fundamental hardness of spanning the search space using first order methods, which we refer to as identifying the global minimizer. 

Global lower bounds have previously been established for the "non adaptive" case where the set of query points are fixed apriori and there exists a bijection from $T$ points to an algorithm \cite{NY83, W89} (for further details, please see Section \ref{AppA1}). In contrast, we consider the "adaptive" case. For instance, where $y_t$ depends on $y_1...,y_{t-1}$ where the bijection does not exist. Besides being adaptive, our lower bounds are a lot tighter, $O(2^d n^{-1/2})$ versus $O(n^{-1/d})$ \cite{NY83} with exponential in the latter holding only for an $\epsilon$-accuracy with $\epsilon<1$ while ours holds for any $\epsilon$ (for further details, please see Section \ref{AppA2}).

Algorithm agnostic lower bounds for the adaptive case typically involve the use of tools such as circuit complexity, Fano's, LeCam's, etc. Such an approach provides an overarching framework that incorporates any possible first order algorithm. The algorithm could be as simple as gradient descent that decides its next point just based on the gradient at its previous point or as complex as using the gradient information from all the points that the algorithm has traversed. Previous work on evaluating lower bounds for stochastic convex optimization has employed a similar framework \cite{Agrawal}. Related work on the non convex front provides conditions under which adaptive local non-convex minimax results in better local convergence rates. The local minimax formulation restricts functions close to an apriori reference function \cite{Singh18}. In contrast, to the best of our knowledge, we are the first to provide adaptive global non-convex minimax lower bounds (for further details, please see Section \ref{AppA3}).. 

In this work, we formulate a parametrized subclass of non convex functions. The parametrization defines the position and the depth of the minimizers of each function. We employ classical statistical minimax bounds to the parameter estimation. Next, we claim that non convex optimization, also viewed as a set reconstruction problem, over the subclass is at least as hard as the parameter estimation. These naturally boil down to prove that it can take exponential time for first order methods to converge to any global minimizer of a non convex function.     

\section{Preliminaries}
In this section, we describe the mapping of non convex optimization to an oracle framework. We define the oracle, its properties and the error for any optimization method $\mathcal{M}$ on a function $f$ at any time step or query to the oracle. We also provide a brief discussion on the mapping of the general minimax risk framework to our setting before concluding with an intuitive picture of the function class of interest. For the convenience of the reader, we present an overview of our notations. 
\subsection{Notation}
We use $\mathcal{S}$ to denote a set of $d$-dimensional points, each denoted by $x \in \mathcal{S}$. We use ${||x||}_p$ to denote the $\ell_p$ norm for $p \in [1,\infty]$. For any two distributions $\mathbb{P}$ and $\mathbb{Q}$, we represent the Kullback-Leibler (KL) divergence between the distributions by $\mathbb{KL}(\mathbb{P}||\mathbb{Q})$. We use $\mathbb{I}$ to denote the Iverson function, i.e., $\mathbb{I}(A)$ refers to a random variable that takes a 0-1 value conditioned on set $A$.  
\subsection{Stochastic first order oracles}
A (stochastic) first order oracle, when queried at a point in its domain, returns the (noisy) function value and the (noisy sub) gradient at that point. We define general first order oracles $\mathcal{O}_{p, \sigma}$ below.  
\begin{definition}
\label{def:oracle}
Given a function class $\mathcal{F}$ on a domain $\mathcal{D}$, the class of first-order stochastic oracles consists of random mappings $\phi:\mathcal{D} \times \mathcal{F} \xrightarrow{} \mathcal{I}$ of the form $(\widehat{g}(x),\widehat{v}(x) )$ such that 
\begin{gather}
\begin{split}
    \mathbb{E}[\widehat{g}(x)] = g(x), \hspace{0.7cm} \mathbb{E}[\widehat{v}(x)] \in \frac{\partial \hat{g}(x)}{\partial x} \hspace{0.5cm} \text{and} \\
    \label{eq:variance}
   \mathbb{E}[||\widehat{v}(x)||_p] \leq \sigma^2 \hspace{2cm}
   \end{split}
\end{gather}
\end{definition}
Note that stochastic first order oracles that satisfy the conditions given by \autoref{eq:variance} return unbiased estimates of both the function value and the gradient with control over the variance of the oracle answers. In this work, the algorithm uses information obtained from oracle answers to queries to internally reconstruct a set $\mathcal{S}$ of minimizers of a function within the subclass.

In the context of non-convex optimization, the algorithm queries the oracle at every time step at a single point in the domain $\mathcal{D}$. The algorithm is allowed the flexibility to use all previous information that it has obtained from the oracle, i.e., $Y = \{\mathcal{I}_{x_{1:T-1}}\}$ to decide its next point $x_{T}$. We use $Y$ to represent all information that the algorithm has queried from the oracle. We note, however, that the indexing with time merely provides an intuitive optimization understanding. An alternate and a more general view of this setting is the maximum information that can be sourced from the oracle in a budget of $T$ queries. The points at which these $T$ queries are made can be completely arbitrary and not necessarily sequential.

\subsection{Non convex optimization in the oracle model}
Non convex optimization is the task of retrieving the global minimizer of a non convex function $g$ over its domain $\mathcal{D}$. In general, any non convex function $g$ has multiple local minimizers. In other words, non convex optimization retrieves $x^* = \argmin_{x \in \mathcal{D}} g(x)$ assuming such a global minimum exists. This is typically done with the help of an optimization method $\mathcal{M}$ that involves iterative sampling from the domain based on information from the previous samples. Common first order optimization algorithms such as stochastic gradient descent use only the information from the previous sample and the gradient at the previous sample to determine the next sample. However, formulating the non-convex optimization problem in an oracle framework aids in the assumption of an algorithm agnostic approach that represents a much more general scenario. It captures any first order algorithm that can be provided with the flexibility to choose the next sample using information from all or some of the previous (noisy) samples and (noisy) gradients in an arbitrary manner as desired. 

We cast non convex optimization as a set reconstruction problem in the oracle model. Intuitively, this is because identifying the global minimizer requires the identification of all possible minimizers of the function. We discuss more elaborately on the same in Section \ref{sec:function_class}. The algorithm queries the oracle for the function values and the gradients at various points and internally reconstructs $\mathcal{S}_T$ using the information $ Y = \{\mathcal{I}_{x_{1:T}}\}$. The elements of the set $\mathcal{S}_T$ could be viewed as the best guess of the possible minimizers of a function in the function class by the algorithm. At any time $T$ or after $T$ queries, we define the error of the optimization algorithm $\mathcal{M}$ based on $\mathcal{S}_T$ as  
\begin{equation}
\label{eq:error}
    \epsilon_T(\mathcal{M}, g, \mathcal{S}_T, \phi) = \sum \limits_{x \in  S_T} g(x) - \inf \limits_x g(x) 
\end{equation}
 Thus, the error term helps evaluate the quality of the algorithm's estimate of the possible minimizers. For oracles that are stochastic, the error given by Equation \eqref{eq:error} in itself is a random variable. In this case, the expectation of the error term over the oracle's randomness given by $\mathbb{E}_\phi[\epsilon_T(\mathcal{M}, g, \mathcal{D}, \phi)]$ is considered. 
 
Next. we discuss the minimax framework applied to general statistical problems in short. The minimax framework allows for a well defined objective to interpret the optimality of algorithms and has been widely used in statistics and machine learning \cite{wainwright_2019, Wasserman}. A general minimax formulation consists of a family of distributions $\mathcal{Q}$ over a sample space $\mathcal{A}$ and a function $\omega : \mathcal{Q} \xrightarrow{} \Omega $. $\omega(\mathcal{Q})$ is treated as a parameter of the distribution $\mathcal{Q}$. We aim to estimate the parameter $\omega(\mathcal{Q})$ from on a set of $m$ observations $A = (o_1, ..., o_m)$ drawn from the (unknown) distribution $\mathcal{Q}$. In order to evaluate the quality of the estimator $\hat{\omega}$, we let $\rho: \Omega \time \Omega \xrightarrow{} \mathbb{R}_+$ denote a pre-metric on the space $\Omega$, which we use to measure the error of an estimator $\hat{\omega}$ with respect to $\omega(\mathcal{Q})$. We note that defining this pre-metric in a suitable way maps different problems to the minimax framework. Thus, for a distribution $Q \in \mathcal{Q}$ and for a given estimator $\hat{\omega} : \mathcal{A}^m \xrightarrow{} \Omega$, we evaluate the quality of the estimate $\hat{\omega}(A)$ in terms of the expected risk
$
    \mathbb{E}_{A \in Q^m}[\rho(\hat{\omega}(A), \omega(Q))]
$. A common approach for choosing an estimator $\hat{\omega}$ is to select one that minimizes the maximum risk \cite{Wald}, given by
$
    \sup \limits_{Q \in \mathcal{Q}}  \mathbb{E}_{A \in Q^m}[\rho(\hat{\omega}(A), \omega(Q))]
$.

An optimal estimator for this pre-metric then gives the minimax risk, defined as 
\begin{equation}
     \inf \limits_{\hat{\omega}}\sup \limits_{Q \in \mathcal{Q}}  \mathbb{E}_{A \in Q^m}[\rho(\hat{\omega}(A), \omega(Q))]
\end{equation}

where we take the supremum (worst case) over the distributions $Q \in \mathcal{Q}$ and infimum over all estimators $\hat{\omega}$.

Finally, we map  general minimax risk to our setting wherein, given a family of functions $\mathcal{F}$ defined over domain $\mathcal{D}$ is optimized using a class of optimization methods $\mathcal{M}$ on a budget of $T$ queries, we define the minimax risk as following
\begin{equation}
 \epsilon_T^*(\mathcal{F}, \mathcal{D}; \phi) = \inf \limits_{M \in \mathcal{M}}  \sup \limits_{g \in \mathcal{F}} \mathbb{E}_\phi[\epsilon_T(\mathcal{M}, g, \mathcal{D}, \phi)]  
\end{equation}
where the supremum is over the family of functions and the infimum over the optimization method. In essence, we are interested in bounding limits on the performance of the best possible algorithm on the hardest possible function in the class. We describe the general function class of interest below followed by the construction of a difficult subclass of functions in the next section. 
\subsection{Function class of interest}
We paint an intuitive picture of the general function class before formally defining a subclass in the following section. Let the space of optimization be $d$-dimensional. Lets assume there exists a function class consisting of $2^d$ functions, each of which have minimizers present in the dimensions dictated by an element (a set) of the power set of $d$ elements. In other words, the identification of the minimizers of a specific function requires the identification of its corresponding element (a set) in the power set. If we are able to define such a function class, we are able to embed an exponential number of functions unique in their set of minimizers in a $d$-dimensional space. In the formal definition, we extend this to embed a super exponential number of minimizers. We represent the broad class of non-convex functions by $\mathcal{F}$. In order to avoid issues with subsets, we formulate an equivalent problem of estimating unique minimizers. Further, we modulate the depth of the function at its minimizers using a random vector $\theta$. Due to the presence of this random vector $\theta$ that modulates the depth of the function at any minimizer, it is not possible for the algorithm to zero in on a global minimizer until all the minimizers of the function have been identified. 

\section{Problem Formulation}
\label{section3}
In this section, we formally introduce the problem by means of a parametrized subclass of functions. Our information-theoretic result relies on the construction of a restricted class of functions. The use of restricted ensembles is customary for information-theoretic lower bounds \cite{Santhanam, Wang_W, Ke}. We first show that if a certain set $\mathcal{S}$ closely reconstructs the minimizers of a function in the subclass, there is no other function in the subclass that $\mathcal{S}$ can minimize. Next, we use Fano's inequality to obtain a lower bound on the parameter estimation which is equivalent to function identification. We then define a hypothesis test to show that non convex optimization is at least as hard as function identification.     

\subsection{Constructing a difficult subclass of functions}
\label{sec:function_class}
To construct the desired subclass of functions with exponential number of minimizers, consider $\mathcal{Z} = \{0, 1\}^d$. The number of unique sets $z \in \mathcal{Z}$ is $2^d$, generally called the cardinality of $\mathcal{Z}$. The subclass of functions $\mathcal{G}$ are designed have their minima at $\{2z-1/2, 2z-1/4\}^d$, the exact permutation decided by $\alpha \in \mathcal{V}$, where $\mathcal{V} \subset \{-1, 1\}^{2^d}$ such that any $\alpha, \beta \in \mathcal{V}$ satisfy
\[
\Delta_\mathbb{H} (\alpha, \beta) \geq 2^d/4
\]

where $\Delta_\mathbb{H} $ denotes the Hamming metric. As the minimizers are defined by elements of set $\mathcal{V}$, it parametrizes the subclass of functions. Thus, for every $\alpha \in \mathcal{V}$, there exists a function in the subclass. From a classical fact, we have cardinality $|\mathcal{V}| \geq (2/\sqrt{e})^{2^d/2}$ \cite{book}. 

In addition, we assume that a set of random vectors $\Theta$ is sampled. The cardinality of $\Theta$ is $2^d$ and each element $\theta \in \Theta$ is a vector of $2^d$ random variables that is associated with a function in the subclass. Each element of the vector $\theta_z$ is sampled from the distribution $(\frac{1}{4}- \frac{\delta}{2}) \hspace{0.1cm}\mathcal{U}[0,1]$, where $\mathcal{U}[0,1]$ represents the uniform distribution in the domain $[0, 1]$. Once $\Theta$ has been sampled, each function is conditioned on a $2^d$ dimensional vector $\theta$. We use $\theta$ to characterize the depth of the function at a particular minimizer. As $\theta$ is a random vector, the algorithm can never be sure that it has identified the global minimizer until it has identified all minimizers. Importantly, we note that the function subclass is conditioned on the set $\Theta$. The function subclass is given by
 \begin{multline}
 \label{eq:function_}
   g_\alpha(x \given \theta_\alpha)= \frac{1}{2^d}\sum \limits_{z \in \mathcal{Z}} (\frac{1}{2} + \alpha_z\delta + \theta_{\alpha, z})f_1(x) \\+  (\frac{1}{2} - \alpha_z\delta -\theta_{\alpha, z})f_2(x)
 \end{multline}
where, 
\begin{equation}
    \begin{split}
        f_1(x) = \min (\norm{  (x - (2z-1)/2 )}_1, c) \\
        f_2(x) = \min (\norm{  (x - (2z-1)/4 )}_1, c)
    \end{split}
\end{equation}

The function class can be visualized as one with inverted pyramids but separated by a constant $c$ such that on summing, the peaks still remain maintaining the non convex characteristic. 
\subsection{Optimizing well is equivalent to function identification} 

We claim that finding the global optimizer is equivalent to the identification of all minimizers. This follows from the fact that as $\theta$ is a random vector, the algorithm can never be sure there is not a lower minima until all of them have been identified. We note that for a given $g_\alpha$ that is parametrized by $\alpha$, the set $\mathcal{S}_\alpha$ is completely defined and is equivalent to the identification of $\alpha$. We show that retrieving the global optima is at least as hard as reconstruction of the set $\mathcal{S}_\alpha$. If the method $\mathcal{M}$ is able to optimize over the function class $\mathcal{G}(\delta)$ upto a desired tolerance, then it is capable of identifying the function $g \in \mathcal{G}$. In order to show this, we define a discrepancy measure to measure \textit{closeness} of two functions in the subclass $\mathcal{G}$ is measured. With two functions $g_\alpha, g_\beta \in \mathcal{G}$ and for any set $\mathcal{S}$, we define 
\begin{multline}
    \label{eq:rho}
    \rho(g_\alpha, g_\beta, \mathcal{S} \given \theta_\alpha, \theta_\beta) = \sum \limits_{x \in \mathcal{S}} g_\alpha(x \given \theta_\alpha) + g_\beta(x \given \theta_\beta) \\  - \inf \limits_{x \in \mathcal{D}} g_\alpha(x \given \theta_\alpha)  - \inf \limits_{x \in \mathcal{D}} g_\beta(x \given \theta_\beta)
\end{multline}

The discrepancy measure $\rho(g_\alpha, g_\beta, \mathcal{S}|\theta_\alpha, \theta_\beta)$ is non negative and symmetric in its arguments. This premetric could be intuitively viewed as the $\ell_1$ functional norm over the set of interest which seems natural in the context of non convex optimization. Using $\rho(g_\alpha, g_\beta, \mathcal{S}|\theta_\alpha, \theta_\beta)$, we define the minimum discrepancy measure between any two functions in the subclass $\mathcal{G}$ by 
\begin{equation}
\label{eq: Psi}
    \Psi(\mathcal{G}(\delta)) = \min \limits_{\alpha \neq \beta \in \mathcal{G}} \inf \limits_{\theta_\alpha, \theta_\beta} \inf \limits_{\mathcal{S}} \rho(g_\alpha, g_\beta, \mathcal{S} \given \theta_\alpha, \theta_\beta )
\end{equation}

In the context of a predefined subset of functions $\mathcal{G}$, we denote $\Psi(\mathcal{G}(\delta))$ as $\Psi(\delta)$. Next, we show that optimization up to a certain tolerance implies the identification of a specific function in the subclass. 

\begin{lemma}
\label{lemma:cond}

For any set $\mathcal{S}$ such that the for any $x_1, x_2 \in \mathcal{S}$ satisfy ${||x_1 - x_2||}_1 > c$ and any  $\theta_\alpha \in \{(0, \frac{1}{4}- \frac{\delta}{2})^{2^d}\}$, there can be at most one function $g_\alpha \in \mathcal{G}$ such that
\begin{equation}
    \sum_{x \in S}g_\alpha(x \given \theta_\alpha) - \inf \limits_{x} g_\alpha(x \given \theta_\alpha) \leq \Psi(\delta)/3
\end{equation}
Thus, if we have a set $\mathcal{S}$ that minimizes a function $g_\alpha \in \mathcal{G}$ up to a certain tolerance ($\Psi(\delta)/3$), then $\mathcal{S}$ cannot approximately minimize any other function in the subclass $\mathcal{G}$. 
\end{lemma}

Proof can be found in Appendix \ref{AppC1}. Now, consider an optimization method $\mathcal{M} \in \mathbb{M}$ which makes a set of queries to the oracle. The method is allowed to make use of this information in any arbitrary manner before it makes the next query. In addition, after $T$ queries, the algorithm reconstructs a set $\mathcal{S}_T$ from the information that it has obtained from queries to the oracle. Assuming that an optimization algorithm is able to reconstruct a set $\mathcal{S}_T$ that minimizes expected error up to a certain tolerance ($\psi(\delta)/9)$, we can then define a hypothesis to identify $\alpha^*$ correctly at least $2/3$ of the time.

\begin{lemma}
\label{lemma:upper}
Suppose that an algorithm $M$, with access to $ Y = \{\mathcal{I}_{x_{1:T}}\} = \phi(x^1:x^T; g_\alpha^* \given \theta)$ obtains an expected error satisfying
\begin{equation*} 
    \mathbb{E}_\Phi[\epsilon_T(\mathcal{M}_T, \Phi, \mathcal{G}, \mathcal{D}] \leq \Psi(\delta)/9 
\end{equation*}
then one is able to construct a hypothesis test $\widehat{\alpha} : \phi(x^1:x^T; g_\alpha^* \given \theta) \rightarrow{} \mathcal{V}$, such that for all $\theta \in \{(0, \frac{1}{4}- \frac{\delta}{2})^{2^d}\}$, we have $\max \limits_{\alpha^* \in \mathcal{V}} \mathbb{P[\widehat{\alpha} \neq \alpha^* \given \theta]} \leq 1/3$
\end{lemma}

Proof can be found in Appendix \ref{AppC2}. Thus, if the optimization method is able to obtain a small enough error, we are able to identify the correct function most of the time.  

\subsection{Oracle answers and coin tossing}
We now show that for the defined function subclass, the stochastic first order oracle answers ($\widehat{f}(x), \widehat{v}(x)$) can be viewed as coin tosses. Specifically, we associate the scaling factor to the bias of a coin for each $z$. Thus, the information obtained from oracle answers can be interpreted as information obtained from flips of $2^d$ coins, with each coin having a bias from the set
\begin{equation}
\kappa(\delta) = \{\{1/2 + \alpha_z\delta + \theta_{\alpha, z} | z \in Z \} | \alpha \in \mathcal{V}, \theta_\alpha \in \Theta  \}    
\end{equation}

We represent the oracle decision process in terms of coin tosses with the following algorithm

1) Pick $\ell$ indices between $\{1, ... , 2^d\}$ without replacement. $\mathcal{Z}_\ell$ represents the corresponding set of $z$'s 

2) Draw $b_z \in \{0, 1\}$ for each $z$ according to a Bernoulli distribution with parameter $1/2 + \alpha_z \delta + \theta
    _{\alpha, z}$
    
3) For a given input $x$, return the value $\widehat{g}_{\alpha}(x \given \theta_\alpha)$ and a sub-gradient $\widehat{v}_\alpha(x) \in \partial \widehat{g}_{\alpha}(x \given \theta_\alpha)$ of the function based on the outcomes of the coin tosses $b_z$
    \begin{equation}
        \widehat{g}_\alpha(x \given \theta) := \frac{1}{\ell} \sum_{z \in \mathcal{Z}_\ell} b_z f_1(x) + (1-b_z) f_2(x)
    \end{equation}

We observe easily that $\mathbb{E}_\phi[\widehat{g}_\alpha(x \given \theta_\alpha)] = g_\alpha(x \given \theta_\alpha)$ is satisfied as required by Equation \eqref{eq:variance}. As the expectation is over the randomness of the oracle, this holds for the sub-gradients as well. 

\subsection{Lower bounds on coin tossing}
Estimating the lower bounds on correctly predicting $\alpha^* \in \mathcal{V}$ from the information obtained from the oracle, thus, boils down to probability of success in estimating a binary vector from coin tosses. 

We note that the subclass of functions are parametrized by the finite set $\mathcal{V}$. However, note that the conditioning of the function subclass on the random vector $\theta \in \{(0, \frac{1}{4}-\frac{\delta}{2})^{|\mathcal{V}|}\}$  modulates the depth at the minimizers in Equation \eqref{eq:function_}. To link this conditioning with Fano's, we use a result that treats this conditioned random variable as a latent variable and provides a bound on information obtained under this conditioning over any arbitrary distribution of $\Theta$ \cite{Ghoshal} (further details in Appendix \ref{AppC3}). In our setting, it involves estimating $\alpha^*$ from answers returned by the oracle $ Y = \{\mathcal{I}_{x_{1:T}}\} = \phi(x^1:x^T; g_\alpha^* \given \theta)$. Thus, by evaluating the mutual information between $\alpha^*$ and $Y$  and upper bounding the same, we upper bound the information from $Y$ that the algorithm could use to estimate its set reconstruction $\mathcal{S}_T$. In other words, this bounds the best $\mathcal{S}_T$ possible by any algorithm.

\begin{lemma}
\label{lemma:lower}
An estimate $\widehat{\alpha}$ of a Bernoulli parameter vector $\alpha^*$ chosen uniformly at random from the packing set $\mathcal{V}$ that is  obtained from the outcome of $\ell \leq 2^d$ coins chosen uniformly at random at each iteration $t= 1, .., T$ satisfies
\begin{align}
\begin{split}
     \sup \limits_{\mathcal{P}_\Theta} \sum_{\alpha \in \mathcal{V}} \int_{\theta_0} Pr\{\alpha \neq \widehat{\alpha} \given \theta = \theta_0\} \mathcal{P}_\Theta(\theta_0)\mathcal{P}_{\mathcal{\alpha}}(\alpha)  \\
     \geq 1 - 2\frac{\ell T (1+2\delta)^2 + \log 2}{2^{d} \log (2/\sqrt{e})}  
          \end{split}
\end{align}

where the probability is taken over coin toss randomness of the oracle and the uniform randomness over the choice of $\alpha^*$
\end{lemma}
\begin{proof}
We introduce notation to represent the $\ell$ coins chosen and set of oracle answers obtained. For $t = 1, ... , T$, $U_t$ represents the subset of $\ell$ coins chosen at each iteration, $X_{t,i}$ represents the outcome of the $z_i^{th}$ coin at time instant $t$ and $Y_{t, i}$ represents a vector of dimension equal to that of $\alpha^*$
  \[
    Y_{t, i} = \left\{\begin{array}{lr}
        X_{t, i} & \text{if } i \in U_t\\
        0 & \text{} otherwise\\
        
        \end{array}\right\}
  \]
The core of the proof deals with the estimation of the mutual information between the information obtained for the queries from the oracle and the true parameter of interest $\alpha^*$, $I(\{U_t, Y_t\}_{t=1}^T, \alpha^* \given \theta)$. By the chain rule for mutual information and properties of the conditional entropy, we have:
\begin{align}
\nonumber
I(\{U_t,Y_t\}_{t=1}^T ; \alpha^* \given \theta) &= \sum \limits_{t=1}^T I((U_t, Y_t) ;\alpha^*\given \{U_{t'},Y_{t'}\}_{t'=1}^{t-1}, \theta) \\ 
\nonumber
& \leq \sum \limits_{t=1}^T I((U_t, Y_t); \alpha^* \given \theta) \\ &\leq T\max_t I((U_t, Y_t);\alpha^* \given \theta)
\end{align}

We consider $I(U_1, Y_1; \alpha^* \given \theta) = \max_t I(U_t, Y_t; \alpha^* \given \theta)$ without loss of generality. The rest of the proof follows directly from Fano's inequality. Thus, it is sufficient to show that $I((U_1, Y_1); \alpha^*\given \theta) \leq  \ell (1+2\delta)^2$ to bound the estimate of the mutual information between the queries and $\alpha^*$. By the chain rule for mutual information,
\begin{equation*}
    I((U_1, Y_1); \alpha^* \given \theta) = I(Y_1; \alpha^* \given U_1, \theta) + I(\alpha^*; U_1 \given \theta)
\end{equation*}
Due to independence between the sampling process of $U_1$ and $\alpha^*$, the second term $I(\alpha^*; U_1 \given \theta) = 0$. We can rewrite the first term from the definition of conditional mutual information as
\begin{equation}
    I(Y_1; \alpha^*| U_1, \theta) = \mathbb{E}_U[\mathbb{KL}(\mathbb{P}_{Y_1| \alpha^*, U_1, \theta} || \mathbb{P}_{Y_1|U_1, \theta})]
\end{equation}
As $\alpha$ is uniformly distributed over $\mathcal{V}$ and from the convexity of KL Divergence, we have \cite{Yu97}
\begin{multline}
\label{eq: Yu}
   \mathbb{KL}(\mathbb{P}_{Y_1|\alpha^*, U_1, \theta}||\mathbb{P}_{Y_1|U_1, \theta})  \\ \leq \frac{1}{|\mathcal{V}|} \sum_{\alpha \in \mathcal{V}} \mathbb{KL}(\mathbb{P}_{Y_1|\alpha^*, U_1, \theta} || \mathbb{P}_{Y_1|\alpha, U_1, \theta}) 
\end{multline}
For any pair $\alpha$, $\alpha^* \in \mathcal{V}$, the summation of the KL Divergence can be bounded by the KL divergence between $\ell$ independent pairs of Bernoulli variates with parameters being $1/2 + \delta + \theta_0$ and $1/2 -\delta-\theta_0$. Thus, we denote the KL Divergence between a single pair of Bernoulli variates with parameters $1/2 + \delta + \theta_0$ and $1/2 - \delta - \theta_0$ by $\mathbb{KL}(\delta, \theta_0)$ given by 
\begin{align}
\nonumber
         \mathbb{KL}(\mathbb{P}_{Y|\alpha^*, U, \theta_0} || \mathbb{P}_{Y|\alpha, U, \theta_0}) \leq \mathbb{KL}(\delta, \theta_0) 
         \end{align}       
\begin{align}
             \begin{split}
             \nonumber
         & =(\frac{1}{2} + \delta + \theta_0) \log \frac{\frac{1}{2} + \delta + \theta_0}{\frac{1}{2}-\delta + \theta_0} \\
        &+ (\frac{1}{2} - \delta - \theta_0) \log \frac{\frac{1}{2} - \delta - \theta_0}{\frac{1}{2}+\delta + \theta_0}\\
                     \end{split}
                     \end{align}
                     \begin{align}
        %
        %
        \nonumber
        &=2 (\delta + \theta_0) (1 + \frac{4\delta + 4\theta_0}{1-2\delta - 2 \theta_0}) \\
        &\leq \frac{8 {(\delta + \theta_0)}^2}{1- 2\delta - 2\theta_0} 
\end{align}       

From Lemma \ref{lemma: Ghoshal} (see Appendix \ref{AppC3}), we take $\sup$ over the latent variable, here $\theta_0$. Thus, we have           
\begin{align}
        \mathbb{KL}(\delta) &\leq \sup_{\theta_0} \frac{8 {(\delta + \theta_0)}^2}{1- 2\delta - 2\theta_0} 
        = \frac{(1+ 2\delta)^2}{1-2\delta} 
         \leq 2(1 + 2\delta)^2
\end{align}
Thus, as long as $\delta \leq 1/4$, from the bound \label{eq:Yu}, we have 
\begin{align}
\mathbb{KL}(\mathbb{P}_{Y_1|\alpha^*, U_1, \theta}||\mathbb{P}_{Y_1|U_1, \theta}) \leq (1+ 2\delta)^2 \ell
\end{align}

Following the proof backwards leads to the desired upper bound $I((U, Y);\alpha^* \given \theta) \leq T\ell (1+2\delta)^2$, thereby completing the proof. 
\end{proof}
\section{Main theorem and Proof}
In this section, we provide our main theorem stating that first order methods take exponential time to span the search space to identify a global minimizer and its proof. 
\begin{theorem}
\label{the: t1}
There exists a universal constant $c_0 > 0$ such that any first order method provided with information from $\ell \leq 2^d$ oracle answers to optimize over the function class $\mathcal{F}_{ncv}(\mathcal{D})$ satisfies the following lower bound
\begin{equation}
    \sup_\phi \epsilon^*(\mathcal{F}_{ncv}, \phi) \geq c_0 \sqrt{\frac{2^d}{T \ell}}
\end{equation}
\end{theorem}
\begin{proof}
We consider an oracle that reveals information based on $\ell$ of the $2^d$ coin tosses with respect to the point with which it has been queried. From Lemma \ref{lemma:lower}, we have the lower bound 
\begin{align}
\begin{split}
         \sup \limits_{\mathcal{P}_\Theta} \sum_{\alpha \in \mathcal{V}} \int_{\theta_0} Pr\{\alpha \neq \widehat{\alpha}| \theta = \theta_0\} \mathcal{P}_\Theta(\theta_0)\mathcal{P}_{\mathcal{\alpha}}(\alpha) \\ \geq 1 - 2\frac{\ell T (1+2\delta)^2 + \log 2}{2^{d} \log (2/\sqrt{e})}
         \end{split}
\end{align}

 Applying the upper bound from Lemma \ref{lemma:upper} requires the expected error to satisfy the conditions under which the upper bound holds, that is, the expected error is required to be within a certain tolerance. In order to evaluate the tolerance, we derive $\Psi(\delta)$ as follows. 

First, we compute $\inf g_\alpha(x)$ which is achieved at the global minima given by, 
  \[
    x^* = \left\{\begin{array}{lr}
        (2z-1)/2 & \text{if } \alpha_z = 1, \text{arg} \sup_{\Tilde{z}} \theta_{\alpha, \Tilde{z}} = z\\
        (2z-1)/4 & \text{if } \alpha_z = -1, \text{arg} \sup_{\Tilde{z}} \theta_{\alpha, \Tilde{z}} = z\\
        
        \end{array}\right\}
  \]
At any such point $x^*$ that minimizes $g_\alpha$, all $z \in \mathcal{Z}$ in the summation apart from the one at which the minimum occurs contribute $(1/2 + \delta)c + (1/2 - \delta)c = c$ while from the $z$ where the minimum occurs, we have a contribution of $(1/2- \delta)c$. Thus, we have 
\begin{equation}
\begin{split}
    \inf \limits_{x} g_\alpha(x)  & = (|\mathcal{Z}|-1)c + (\frac{1}{2} -\delta - \sup_z \theta_{\alpha, z})c \\ &= (|\mathcal{Z}|-1)c + (\frac{1}{4} - \frac{\delta}{2})c
    \end{split}
\end{equation}
We note that $\rho$ and $\Psi$ are defined over sets $\mathcal{S}$, which contain the algorithm's estimates of the minima of a function in the function class $\mathcal{G}$. The error is computed as a sum over all the elements in the set $\mathcal{S}$. In other words, any point in the set that is incorrectly identified as a minimizer adds to the error and correspondingly to the discrepancy measure.    

Consider two functions $g_\alpha, g_\beta \in \mathcal{G}$ and the set $\mathcal{S}$. From Equation \ref{eq:rho}, for a specific $x \in \mathcal{S}$
\begin{align*}
           \rho(g_\alpha, g_\beta, \mathcal{S}, x) = g_\alpha(x) + g_\beta(x) - \inf \limits_{x \in \mathcal{D}} g_\alpha(x) - \inf \limits_{x \in \mathcal{D}} g_\beta(x)
\end{align*}
We first estimate  $g_\alpha(x) + g_\beta(x) $ as follows
\begin{equation*}
      g_\alpha(x) + g_\beta(x) \\
     \end{equation*}
       \begin{align}
       \nonumber
\begin{split} 
     &= \frac{1}{2^d}\sum\limits_{z \in \mathcal{Z}} (1 + \alpha_z\delta + \beta_z\delta + \theta_{\alpha, z} + \theta_{\beta, z}) f_1(x, z)\\[1ex]
     &\phantom{ = {}}+ (1 - \alpha_z\delta - \beta_z\delta - \theta_{\alpha, z} - \theta_{\beta, z})f_2(x, z) \\[1ex]
     \end{split} \\
     \begin{split}&= \frac{1}{2^d}\sum\limits_{z \in \mathcal{Z}}[(1 + \theta_{\alpha, z} + \theta_{\beta, z})f_1(x,z) \\ &\phantom{ = {}}+ (1 - \theta_{\alpha, z} - \theta_{\beta, z})f_2(x,z)] \mathbb{I}(\alpha_z\neq \beta_z)\\
     &\phantom{ = {}}+  [(1+2\alpha_z\delta + \theta_{\alpha, z} + \theta_{\beta, z})f_1(x,z) \\  &\phantom{ = {}}+ (1-2\alpha_z\delta- \theta_{\alpha, z} - \theta_{\beta, z})f_2(x,z)] \mathbb{I}(\alpha_z = \beta_z) \\
\end{split}
\end{align}

In order to better parse the summation, we denote the contribution for each $z$ as $\{g_\alpha(x) + g_\beta(x)\}_z$. 

For $\alpha_z \neq \beta_z$, we have

   $\{g_\alpha(x) + g_\beta(x)\}_z =$
   
      $\left\{\begin{array}{lr}
         c(1+\theta_{\alpha, z} + \theta_{\beta, z}) & \text{if } x = (2z-1)/2\\
         c(1-\theta_{\alpha, z} - \theta_{\beta, z}) & \text{if } x = (2z -1)/4\\
         2c & \text{if } x \neq  (2z-1)/2, (2z -1)/4\\
        \end{array}\right\}$

 For $\alpha_z = \beta_z$, we have
 
    $\{g_\alpha(x) + g_\beta(x)\}_z = $ 
    
        \resizebox{\columnwidth}{!}{%
    $\left\{\begin{array}{lr}
         (1-2\delta - \theta_{\alpha, z} - \theta_{\beta, z})c & \text{if } \alpha_z = \beta_z =1, x = (2z-1)/2\\
         (1-2\delta - \theta_{\alpha, z} - \theta_{\beta, z})c & \text{if } \alpha_z = \beta_z =-1, x = (2z-1)/4\\
         (1+2\delta + \theta_{\alpha, z} + \theta_{\beta, z})c & \text{if } \alpha_z = \beta_z =1, x = (2z-1)/4\\
         (1+2\delta + \theta_{\alpha, z} + \theta_{\beta, z})c & \text{if } \alpha_z = \beta_z =-1, x = (2z-1)/2\\
         2c & \text{if } \alpha_z = \beta_z , x \neq  \{(2z-1)/2,  (2z-1)/4\}\\
        \end{array}\right\}$
}

Putting these together lead to multiple cases which we will explore in detail. Note that we are interested in lower bounding $\Psi(\delta)$. Since we do not have control over the selection of $\theta_\alpha$ and $\theta_\beta$, we develop a lower bound for the worst case. We refer to ``common minimizer" and ``unique minimizer" loosely for intuitive understanding. These could be viewed as a minimizer present at that particular $x$ for both the functions but the function value may be different due to its dependence on $\theta$. 

1) \textit{If $x$ is a ``common minimizer"}:  The first two possible cases with $\alpha_z = \beta_z$ are satisfied. Considering $\inf_\Theta$, we have a contribution of $\{g_\alpha(x) + g_\beta(x)\}_z = c(1/2 - \delta)$ from this $z$ and $|\mathcal{Z} -1|2c$ from all other $z$'s. The resulting sum would cancel out $\inf_x g_\alpha + \inf_x g_\beta $ leading to $\rho = 0$.

2) \textit{If $x$ is a ``unique minimizer"}: The first two possible cases with $\alpha_z \neq \beta_z$ are satisfied. Considering $\inf_\Theta$, we have a contribution of $\{g_\alpha(x) + g_\beta(x)\}_z = c(1/2 + \delta)$ from this $z$ and $|\mathcal{Z} -1|2c$ from all other $z$'s. The $|\mathcal{Z} -1|2c$ from all other $z$'s exactly cancels with the first term in $\inf_x g_\alpha + \inf_x g_\beta$. We are now left with $\rho = c(1/2 +\delta) - c(1/2-\delta) = 2\delta c$.

3) \textit{If $x$ is not a ``minimizer"}: We have a contribution of $2c$ from all $z$'s. $|\mathcal{Z}|-1$ of these terms cancel with the first term of $\inf_x g_\alpha + \inf_x g_\beta$ leaving behind $\rho = 2c - c(1/2-\delta) = c(3/2+\delta)$

Now summing over all $x \in \mathcal{S}$, which is the algorithm's guess of the minimizers or the internal reconstruction, the discrepancy measure has at least $2c\delta$ contribution from the $\Delta_\mathcal{H} (\alpha, \beta)$ terms. Thus, we have
\begin{equation*}
    \rho(g_\alpha, g_\beta, \mathcal{S} \given \theta_\alpha, \theta_\beta ) \geq \frac{2\delta c}{2^d} \Delta_\mathcal{H} (\alpha, \beta)
\end{equation*}
Using the above insights, we now evaluate $\Psi(\delta)$ as follows
\begin{equation}
\begin{split}
    \Psi(\delta) & = \min \limits_{\alpha \neq \beta \in \mathcal{G}} \inf \limits_{\mathcal{S}} \rho(g_\alpha, f_\beta)  \geq \frac{2\delta c}{2^d} \frac{2^d}{4}  = \frac{\delta c}{2} 
\end{split}
\end{equation}

Finally, we set $\epsilon = c\delta /18$ satisfying the requirement of $\epsilon \leq \Psi(\delta)/9 $ required to apply Lemma 2. We choose $c =1/8$, considering the minimum possible separation between two different minimizers, which results in $\epsilon = \delta/ 144$ and in this regime, the following holds
\begin{equation}
    1/3 \geq  1 - 2\frac{\ell T (1+2\delta)^2 + \log 2}{2^{d} \log (2/\sqrt{e})}
\end{equation}
Replacing $\delta = 144\epsilon$ and rearranging leads to 
\begin{equation}
    T \geq \Omega \left( \frac{2^d}{\ell \epsilon^2} \right) \\
\end{equation}

\end{proof}
Thus, we have proved that the reconstruction of a set that has a low error on the function subclass requires an exponential number of queries. As set reconstruction is equivalent to identification of all minimizers, it boils down to our claim that identification of global minimizer takes exponential time.

\section{Discussion}

\subsection{Comparison with non adaptive queries \cite{NY83, W89}}
\label{AppA1}
In this section, we define and differentiate between non adaptive evaluations considered in \cite{NY83, W89} and adaptive evaluations for which our result holds. 

\begin{definition}[Non adaptive queries] The queries $x_1, \dots, x_T$ on the domain $\mathcal{X}$ are fixed apriori for all functions $f \in \mathcal{F}$ and the oracle returns (noisy) function values $y_1, \dots, y_T$. Thus, we could denote the queries as a $x_1(\mathcal{F}), \dots, x_T(\mathcal{F})$. 
\end{definition}

The notation in the following Definition \ref{def:2} uses sequential representation for intuitive and clear optimization interpretation. However, the results are much more general (please see Lemma \ref{lemma:lower}) and apply to any adaptive method that include branch and bound, divide and conquer, dynamic programming, or any other conceivable method.

\begin{definition}[Adaptive (sequential) queries] The queries $x_1, \dots, x_T$ on the domain $\mathcal{X}$ for each function $g \in \mathcal{G}$ can be chosen as $
    \{x_1, y_1, \dots, x_{i-1}, y_{i-1}\} \rightarrow{} x_i $
and the oracle returns (noisy) function values $y_1, \dots, y_T$. Thus, we could denote the queries as  $x_1(g), \dots, x_T(g)$.
\label{def:2}
\end{definition}

The result for the non adaptive case is obtained by constructing a function approximation on $T$ non adaptive queries in the $\ell_2$ norm  \cite{NY83, W89}, where the evaluation points in the domain are fixed aprioiri for a given function class $\mathcal{F}$, that is, $x_1(\mathcal{F}), \dots, x_T(\mathcal{F})$. The bound in \cite{NY83, W89} is obtained by picking a single function $f \in \mathcal{F}$ and performing an $\inf$ over a pre defined grid of points in an algebraic manner. It is straightforward that for this lower bound to hold for the function class $\mathcal{F}$, the points chosen must remain the same. This is because picking a different set of points for a different function $f' \in \mathcal{F}$ would change the result. Thus, the non adaptive requirement is embedded in the result.  

On the other hand, our framework allows for each function $g \in \mathcal{G}$ to be queried on a different set of points $x_1(g), \dots, x_n(g)$ using the information gathered about $g$ from the oracle answers to previous queries. That is, we bound the maximum amount of information about a function $g \in \mathcal{G}$ that can be accumulated from a budget of $T$ queries using any conceivable first order algorithm. The \textit{adaptiveness} is captured in the evaluation of mutual information within our information theoretic framework (please see Lemma \ref{lemma:lower}) while only a non adaptive grid is considered in \cite{NY83, W89}. 

Finally, in the non adaptive scenario, there exists a bijection from the query points to an algorithm. However, in the adaptive scenario, such  a bijection from the query points to algorithms does not exist and requires information theoretic approaches such as this work.

\subsection{Comparison for $\epsilon$-accuracy \cite{NY83, W89}}
In this section, we show that our bounds are a) a lot tighter that those obtained for the non adaptive case and b) hold for any desired accuracy $\epsilon$. The prior result obtained for an $\epsilon$-accuracy with $T$ non adaptive queries is $T = \Omega(\epsilon^{-d})$ in \cite{NY83, W89} or equivalently $
\epsilon = \Omega(T^{-1/d}) $
while our Theorem \ref{the: t1} for $T$ adaptive queries is 
$
    \epsilon = \Omega(\sqrt{\frac{2^d}{T}})$
We have $\log_2\epsilon = d-\frac{1}{2} \log_2 T$ versus $\log_2\epsilon = -\frac{1}{d}\log_2T$. Given the above, ours is clearly a lot tighter. Furthermore, the bound in \cite{NY83, W89} is not increasing with respect to $d$ for $\epsilon \geq 1$
\label{AppA2}

\subsection{Comparison with local minimax \cite{Singh18}}
\label{AppA3}
While both \cite{Singh18} and ours are based on adaptive function evaluations, the former proposes a local minimax framework while we propose a global minimax framework. The local minimax framework considers a function class that is similar to a reference function class. They show that the bounds under the above assumption are similar to those of convex optimization. In contrast, we do not incorporate such an assumption as we are interested in establishing the fundamental hardness of non convex optimization. Our result complements the above as the assumption of requiring a reference function class with full knowledge that is similar to the one being optimized over is difficult to incorporate apriori. To make this more concrete, we present the local minimax formulation below. 

\begin{lemma}
\cite{Singh18} For a true underlying function $f \in \Theta'$, a reference function $f_0 \in \Theta$ and a constant $C_1$, with $\epsilon_n(f_0)$ depicting the range in which the local minimax rates apply to be of the same order as convergence rate $R_n(f_0)$, the local minimax estimator $\widehat{x}_n$ that has complete knowledge of the reference function $f_0$ is defined as 
\[
    \inf \limits_{\widehat{x}_n} \sup \limits_{f \in \Theta', \| f- f_0\|_\infty \leq \epsilon_n(f_0)} \Pr\limits_{f} [\mathfrak{L}(\widehat{x}_n; f) \geq C_1 R_n]
\]
where 
$
    \mathfrak{L}(\widehat{x}_n; f) = f(x) - \inf \limits_x f(x) 
$
\end{lemma}

The local minimax result above considers a supremum only over functions $f$ that are $\epsilon_n(f_0)$ - away from a reference function $f_0$, i.e., $\|f - f_0\|_\infty \leq \epsilon_n(f_0)$.
Our global minimax result considers a supremum over all functions $f \in \mathcal{F}$.
This is arguably a more challenging scenario since it considers a supremum over a bigger set.

\section{Conclusion}
We provide lower bounds for non convex optimization for up to first order. Importantly, we note that there is no restriction on the queries being sequential or related in any manner and the function is not required to be differentiable. Thus, we obtain algorithm agnostic lower bounds. Our future work includes exploring better rates with access to second order information or imposing the function class to be differentiable. On the other hand, situations with constraints on queries to be sequential in nature or using distributed optimization provide interesting avenues to understand if they lead to worse rates and by what factors.

\bibliography{references}

\newpage
\appendix
\onecolumn

\begin{center}
\textbf{\Large Supplementary Material: First Order Methods take Exponential Time to Converge to Global Minimizers of Non-Convex Functions}
\end{center}

\section{Detailed Proofs}
\label{sec: AppC}
\subsection{Proof of Lemma \ref{lemma:cond}}
\label{AppC1}
\begin{proof}
From Equations \ref{eq:rho} and \ref{eq: Psi} and given that there is an $\mathcal{S}$ and $\alpha$ such that $\sum_\mathcal{S}g_\alpha - \inf \limits_x g_\alpha \leq \Psi(\delta)/3$, we have for any $\beta \neq \alpha$, 
\begin{align}
    \nonumber
    \Psi (\delta) \leq \rho(g_\alpha, g_\beta, \mathcal{S} \given \theta_\alpha, \theta_\beta) &= \sum \limits_{x \in \mathcal{S}} g_\alpha(x \given \theta_\beta) - \inf \limits_{x} g_\alpha(x \given \theta_\alpha) + g_\beta(x \given \theta_\beta) - \inf \limits_{x} g_\beta(x \given \theta_\beta) \\
    &\begin{aligned}
        \nonumber
        &\leq \sum \limits_{x \in \mathcal{S}} (g_\alpha(x \given \theta_\alpha) - \inf \limits_{x} g_\alpha(x \given \theta_\alpha)) +  \sum \limits_{x \in \mathcal{S}} (g_\beta(x \given \theta_\beta) - \inf \limits_{x} g_\beta(x \given \theta_\beta)) 
    \end{aligned} \\
    &\begin{aligned}
     &\leq \Psi(\delta)/3 +    \sum \limits_{x \in \mathcal{S}} (g_\beta(x \given \theta_\beta) - \inf \limits_{x} g_\beta(x\given \theta_\beta))
     \end{aligned}
\end{align}
Rearranging, we have, 
\begin{equation}
\sum \limits_{x \in \mathcal{S}} (g_\beta(x \given \theta_\beta) - \inf \limits_{x} g_\beta(x\given \theta_\beta)) \geq 2\Psi(\delta)/3 
\end{equation}
\end{proof}

\subsection{ Proof of Lemma \ref{lemma:upper}}
\label{AppC2}
\begin{proof}
We build an estimator of the true $\alpha^*$, denoted as $\widehat{\alpha}$. If  there exists an $\alpha \in \mathcal{V}$ such that
\begin{align*}
    \sum_{x \in \mathcal{S}_T} g_\alpha(x \given \theta_\alpha) - \inf g_\alpha(x \given \theta_\alpha) \leq \Psi(\delta)/3 
    \end{align*}
then we assign $\widehat{\alpha}(\mathcal{S}_T)$ to $\alpha$. In any other case, we pick $\widehat{\alpha}$ at random. From the above hypothesis, we have that
\begin{align*}
   \mathbb{P}[\alpha \neq \widehat{\alpha} \given \Theta] &\leq  \mathbb{P}_\Phi[\sum_{x \in S_T}g_\alpha(x \given \theta_\alpha) - \inf_{x} g_\alpha(x \given \theta_\alpha) \geq \Psi(\delta)/3] \leq 1/3
\end{align*}
    where the upper bound is by the Markov's inequality under the expected error assumption. As this is true for any $\theta$, maximizing over $\alpha$ completes the proof.   
\end{proof}
 
  \begin{figure}[ht!]
    \centering
    \includegraphics[width=0.25\textwidth]{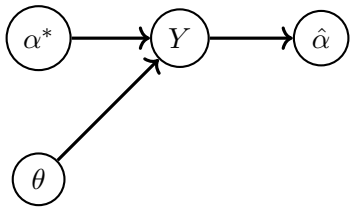}
       \caption{The estimation problem, represented as a Bayesian network}
        \label{fig:Ghoshal}
  \end{figure}

\subsection{Result on Fano's with conditioning \cite{Ghoshal}}
\label{AppC3}
The result stems from extending Fano's inequality to a slightly more general scenario as shown in Figure $\ref{fig:Ghoshal}$. Adopting it to our setting involves estimating $\alpha^*$ from answers returned by the oracle $ Y = \{\mathcal{I}_{x_{1:T}}\} = \phi(x^1:x^T; g_\alpha^* \given \theta)$

\begin{lemma}
\label{lemma: Ghoshal}
\cite{Ghoshal} Let $\alpha \in \mathcal{V}$ and $\theta$ be random variables and let $\widehat{\alpha}$ be any estimator of $\alpha^*$ obtained from samples $Y$. If the random variables $\alpha$ and $\theta$ are independent, then 
\begin{align}
\label{eq: Ghoshal}
    \sup \limits_{\mathcal{P}_\Theta} \sum_{\alpha \in \mathcal{V}} \int_{\theta_0} Pr\{\alpha \neq \widehat{\alpha} \given \theta = \theta_0\} \mathcal{P}_\Theta(\theta_0)\mathcal{P}_{\mathcal{\alpha}}(\alpha) \geq 1 - \frac{\sup_{\theta_0 \in \Theta} I(\alpha^*; Y \given \theta = \theta_0) + \log 2}{\log |\mathcal{V}|}
\end{align}
\end{lemma}

Thus, it is sufficient for us to show the RHS, that is, Fano's inequality in the conditioned case with the supremum applied over $\theta_0$ while estimating the mutual information which we derive in Lemma \ref{lemma:lower}. We note that the algorithm uses $Y$ to internally reconstruct the set $\mathcal{S}_T$.

\end{document}